\documentclass{article}

\usepackage{bbm}
\usepackage{eqnarray,amsmath,amsfonts,amsthm,mathrsfs}
\usepackage{color}
\usepackage{bm}
\usepackage{amssymb}
\usepackage{epsfig}
\usepackage{epsf}
\usepackage{graphicx}
\usepackage{float}
\usepackage{subfigure}
\usepackage{epstopdf}
\usepackage{amsfonts,amsmath,amsthm,amssymb,graphicx,float,fancyhdr,multirow,hyperref}
\usepackage{booktabs,longtable,authblk}
\usepackage{mathrsfs,hhline}
\usepackage{algorithm}
\usepackage{algorithmic}

\usepackage{caption}
\usepackage{fancyhdr}
\usepackage[top=2.5cm, bottom=2.5cm, left=3cm, right=3cm]{geometry}
\usepackage{graphics}
\newtheorem{theorem}{Theorem}[]
\newtheorem{assumption}[]{Assumption}
\newtheorem{definition}[]{Definition}
\newtheorem{example}[]{Example}
\newtheorem{lemma}[]{Lemma}

\newtheorem{remark}[]{Remark}

\newcommand{\T}{\top}

\hypersetup{ colorlinks=true, linkcolor=black, filecolor=black, urlcolor=black }

\title{\LARGE Mirror Descent on Riemannian Manifolds\thanks{Authors are listed in alphabetical order and contributed equally to this work. The corresponding author is Lei Shi. The work of Lei Shi is supported by the National Natural Science Foundation of China [Grant No.12171093]. Email addresses: jxjiang20@fudan.edu.cn (J. Jiang), leishi@fudan.edu.cn (L. Shi), jiyuantan19@gmail.com (J. Tan).}}
\author[1]{Jiaxin Jiang}
\author[1,2]{Lei Shi}
\author[1]{Jiyuan Tan}
\affil[1]{School of Mathematical Sciences, Fudan University, Shanghai 200433, China}
\affil[2]{Shanghai Key Laboratory for
	Contemporary Applied Mathematics, Fudan University, Shanghai 200433, China}
\date{}

\begin{document}
	
	%\date{\today}
	%\address{Address}
	%\email{example@mail.com}
	\maketitle

	\begin{abstract}
Mirror Descent (MD) is a scalable first-order method widely used in large-scale optimization, with applications in image processing, policy optimization, and neural network training. This paper generalizes MD to optimization on Riemannian manifolds. In particular, we develop a Riemannian Mirror Descent (RMD) framework via reparameterization and further propose a stochastic variant of RMD. We also establish non-asymptotic convergence guarantees for both RMD and stochastic RMD. As an application to the Stiefel manifold, our RMD framework reduces to the Curvilinear Gradient Descent (CGD) method proposed in \cite{wen2013feasible}. Moreover, when specializing the stochastic RMD framework to the Stiefel setting, we obtain a stochastic extension of CGD,  which effectively addresses large-scale manifold optimization problems.
	\end{abstract}
%Moreover, we demonstrate that the classical geodesic gradient descent algorithm is a special case in our proposed framework.
%\begin{keywords}
%	Mirror descent, Riemannian optimization, Convergence analysis
%\end{keywords}

{\textbf{Keywords:} Mirror descent, Riemannian optimization, Reparameterization, Non-asymptotic convergence analysis, Stiefel manifolds}

{\textbf{MSCcodes:} 65K05, 90C06, 90C30}

%65K05: Mathematical programming {Algorithms; for theory see 90Cxx]
%90C06: Large-scale problems
%90C30: Nonlinear programming
%\begin{MSCcodes}
%	65K05, 90C06, 90C30
%\end{MSCcodes}

\section{Introduction}
In this paper, we study the Riemannian optimization problem
\[
\min_{x\in\mathcal{M}} f(x),
\]
where $\mathcal{M}$ is a Riemannian manifold, and $f:\mathcal{M}\rightarrow\mathbb{R}$ is a smooth function on $\mathcal{M}$. This problem can be viewed as constrained nonlinear optimization with a manifold constraint. The manifold structure enables us to design efficient algorithms and perform detailed analysis.

Riemannian optimization finds wide applications in practice. For instance, it can be applied to dictionary learning \cite{cherian2016riemannian} and matrix decomposition \cite{vandereycken2013low,tan2014riemannian}. In deep learning, imposing orthogonal constraints on weight matrices can enhance training stability and improve generalization performance \cite{arjovsky2016unitary,bansal2018can}. Such constraints can be modeled by optimizing objective functions over the Stiefel manifold.

One line of research in this field focuses on generalizing unconstrained optimization algorithms to Riemannian manifolds. For example, Riemannian Gradient Descent \cite{zhang2016first}, Riemannian Proximal Point \cite{bento2017iteration}, Riemannian Momentum \cite{alimisis2021momentum}, and Riemannian Stochastic Variance Reduction Gradient \cite{zhang2016riemannian,sato2019riemannian} have been explored. However, Mirror Descent (MD), an important class of algorithms in Euclidean space, has been surprisingly unexplored in the existing literature. This paper takes the first step to fill this gap. 

The MD algorithm was proposed by Nemirovski and Yudin \cite{nemirovski1983problem} as a generalization of gradient descent and was later popularized by the widely used formulation of Beck and Teboulle \cite{beck2003mirror}. 
Consider the Euclidean optimization problem $\min_{x\in\Omega} f(x)$, where $\Omega\subseteq\mathbb{R}^d$ is a closed convex feasible set. 
Let $\psi:\Omega\to\mathbb{R}$ be a differentiable and strongly convex function, often called a potential function. 
The associated Bregman divergence \cite{bregman1967relaxation} is defined as
\[
D_{\psi}(x,y)=\psi(x)-\psi(y)-\langle\nabla\psi(y),x-y\rangle.
\]
Given a step size $\eta_t>0$, the MD update is
\[
x_{t+1}=\arg\min_{x\in\Omega}\left\{\left\langle \eta_t\nabla f(x_t),\,x-x_t\right\rangle+D_\psi(x,x_t)\right\},
\]
where the Bregman term can be interpreted as a geometry-adapted regularizer. 
When $\Omega=\mathbb{R}^d$, the first-order optimality condition yields the implicit dual update
\[
\nabla\psi(x_{t+1})=\nabla\psi(x_t)-\eta_t\nabla f(x_t).
\]
Different choices of $\psi$ recover algorithms suited to different geometries. 
For instance, taking $\psi(x)=\tfrac12\|x\|_2^2$ reduces MD to the standard gradient descent update. When $\Omega$ is the probability simplex, the entropy potential $\psi(x)=\sum_{i=1}^d x_i\log x_i$ (with the convention $0\log 0=0$) leads to Entropy Mirror Descent, and $D_\psi$ becomes the Kullback--Leibler divergence \cite{nemirovski2009robust}. 
By adapting to the problem geometry with mild dependence on dimension, MD has found widespread applications, including image processing \cite{beck2003mirror}, policy optimization \cite{yang2022policy,zhan2023policy}, and neural network training \cite{sun2022mirror,soh2023mirror}.

Subsequent work further broadened the study of MD, with particular emphasis on its extensions and connections to other methods. Nemirovski and Juditsky \cite{nemirovski2009robust} considered using MD to solve stochastic optimization and stochastic saddle-point problems. Duchi et al. \cite{duchi2010composite} generalized these results to composite functions. Recent research has focused on providing alternative explanations for MD. Gunasekar et al. \cite{gunasekar2021mirrorless} demonstrated that MD can be interpreted as the gradient flow on the manifold as the step sizes tend to zero. Amid and Warmuth \cite{amid2020reparameterizing} observed that MD becomes equivalent to gradient descent under different parameterizations. Li et al. \cite{li2022implicit} utilized this observation to explain the implicit regularization in over-parameterized models. Lei and Zhou \cite{lei2020convergence} studied online mirror descent in a setting where data arrive sequentially.

The literature on Riemannian optimization is extensive. We briefly review several related results. 
Zhang and Sra \cite{zhang2016first} establish the first non-asymptotic convergence rates of geodesic gradient descent in both deterministic and stochastic settings. 
Later, Bento et al. \cite{bento2017iteration} analyzed Riemannian subgradient and proximal point methods; however, their results rely on the assumption that the underlying manifold is Hadamard. 
More recently, Srinivasan and Wirth \cite{srinivasan2022sufficient} extended this non-asymptotic convergence analysis to general Riemannian manifolds by constructing a suitable potential function. 
Alimisis et al. \cite{alimisis2021momentum} studied momentum-based acceleration of first-order methods on manifolds. 
In the stochastic setting, Zhang and Sra \cite{zhang2016riemannian} and Sato et al. \cite{sato2019riemannian} applied variance-reduction techniques and derived improved complexity bounds for finite-sum problems. 
Tripuraneni et al. \cite{tripuraneni2018averaging} generalized Polyak--Ruppert averaging to the manifold setting and showed that the resulting averaged iterates enjoy improved convergence guarantees for stochastic gradient descent on manifolds. 
Recently, Shu et al. \cite{shu2025revisit} introduced a quasilinearization framework to analyze proximal-based methods on Hadamard manifolds.
\par

Our contributions can be summarized as follows:

\begin{itemize}
    \item We generalize Euclidean Mirror Descent to optimization on Riemannian manifolds by developing a Riemannian Mirror Descent (RMD) framework via reparameterization. Building on the proposed framework, we further propose a stochastic variant of RMD (stochastic RMD) suitable for large-scale settings in which only stochastic gradient estimates are available.
    \item We establish non-asymptotic convergence guarantees for both RMD and stochastic RMD. Let $T$ denote the total number of iterations. For geodesically convex objectives, we prove that the objective error decays at a sublinear rate of $\mathcal{O}(1/T)$ in the deterministic setting and $\mathcal{O}(1/T^{1/3})$ in the stochastic setting. For smooth nonconvex objectives, we show that the average squared gradient norm decays at a rate of $\mathcal{O}(1/T)$ in the deterministic case and $\mathcal{O}(1/T^{1/2})$ in the stochastic case. Our theoretical results fill the gap of \cite{amid2020reparameterizing,gunasekar2021mirrorless}.
    \item  As an application to optimization on the Stiefel manifold, we show that our RMD framework recovers the Curvilinear Gradient Descent (CGD) method of \cite{wen2013feasible}. Moreover, specializing stochastic RMD to the Stiefel setting yields a stochastic extension of CGD, which we term Stochastic Curvilinear Gradient Descent (SCGD). The proposed SCGD is scalable and computationally efficient for large-scale manifold optimization problems and may be of independent interest. We further validate its effectiveness through numerical experiments.
\end{itemize}

\textbf{Outline of this paper.} In Section~\ref{sec:pre}, we introduce the notation and assumptions used throughout the paper. In Section~\ref{sec:algo}, we present the framework of RMD and Stochastic RMD, and provide several concrete instantiations. In Section~\ref{sec:thm}, we analyze the convergence rates of our algorithms. In Section~\ref{sec:num}, we demonstrate the effectiveness of the proposed methods through numerical experiments. We conclude with a discussion and future research directions in Section~\ref{sec:con}.

\section{Notation and Theoretical Background}\label{sec:pre}

We first recall some basic notions in Riemannian geometry.

\begin{definition}
A Riemannian manifold $(\mathcal{M},g)$ is a smooth manifold endowed with Riemannian metric $g$. Let $\mathcal{T}_{x}\mathcal{M}$ be the tangent space of $\mathcal{M}$ at point $x$. The tangent bundle of $\mathcal M$ is $\mathcal T\mathcal M= \cup_{x\in\mathcal M} \mathcal T_x\mathcal M$. A Riemannian metric $g$ assigns to each $x\in\mathcal M$ an inner product $ g_{x}( \cdotp ,\cdotp ) =\langle \cdotp ,\cdotp \rangle _{x}$, which varies smoothly with $x$. The induced norm on $\mathcal{T}_x\mathcal M$ is denoted by $\|\cdot\|_x$.
\end{definition}
Differentials and gradients are indispensable for optimization on manifolds.
\begin{definition}
Let $f:\mathcal M\to\mathbb R$ be differentiable. The differential of $f$ at $x\in\mathcal M$ is the linear map $ Df( x):\mathcal{T}_{x}\mathcal{M}\rightarrow \mathbb{R}$. The Riemannian gradient $\nabla f(x)\in \mathcal{T}_x\mathcal M$ is the unique vector satisfying
\begin{equation*}
Df( x)[ v] =\langle \nabla f( x) ,v\rangle _{x},\quad \forall v\in \mathcal{T}_{x}\mathcal{M} .
\end{equation*}
\end{definition}

For a differentiable map $F:\mathcal M\to\mathcal N$ between manifolds, we denote by $DF(x): \mathcal{T}_x\mathcal M \to \mathcal{T}_{F(x)}\mathcal N$ for its differential at $x\in \mathcal M$, and write $DF(x)[v]\in T_{F(x)}\mathcal N$ for its action on $v\in \mathcal{T}_x\mathcal M$. We also denote by $\nabla$ the Levi--Civita connection on $(\mathcal M,g)$. In what follows, we will mainly use it in the form of the covariant derivative along a smooth curve.

\begin{definition}
For a smooth curve $ \gamma :[ 0,1]\rightarrow \mathcal{M}$, we write $ \dot{\gamma} (t)$ for its tangent vector field. 
The covariant derivative $\nabla_{\dot{\gamma}}$ is denoted by $\frac{d}{dt}$, as usual. Thus the twice covariant derivative is $\frac{d^2}{dt^2}$.
A curve $\gamma$ is called a geodesic if $ \frac{d^{2} \gamma ( t)}{dt^{2}} =0,\forall t\in [ 0,1]$.
\end{definition}
In this paper, we assume that $(\mathcal M,g)$ is connected and complete (in the Riemannian sense). Equivalently, with the distance function $d_{\mathcal M}$ induced by Riemannian metric, $(\mathcal M,d_{\mathcal M})$ is a complete metric space, or every geodesic can be extended for all time, i.e., the exponential map is defined on $\mathcal{T}\mathcal{M}$. By the Hopf–Rinow theorem, for any $x,y\in\mathcal M$, there exists a minimizing geodesic segment $\gamma:[0,1]\to\mathcal M$ joining them, satisfying $\gamma(0)=x, \gamma(1)=y$, and the length of $\gamma$ is equal to $ d_{\mathcal{M}}( x,y)$.
But this minimizing geodesic $\gamma$ is not necessarily unique.

\begin{definition}
For any $ x\in \mathcal{M}$, the exponential map $\mathrm{Exp}_{x} :\mathcal{T}_{x}\mathcal{M}\rightarrow \mathcal{M}$ satisfies that for any $ v\in \mathcal{T}_{x}\mathcal{M} $, $ \gamma ( t) =\mathrm{Exp}_{x}( tv) ,t\in [ 0,1]$ is a geodesic and the length of $ \gamma$ is $ \Vert v\Vert _{x}$. 
\par For $x\in\mathcal M$ and $r>0$, we denote by $\mathcal N_r(x)=\{y\in\mathcal M:d_{\mathcal M}(x,y)\leqslant r\}$ the closed geodesic ball of radius $r$ centered at $x$, and by $B_r(x)=\{v\in \mathcal{T}_{x}\mathcal{M} : \|v\|\leqslant r\}$ the ball centered at origin of the tangent space. Let $\mathrm{inj}(\mathcal{M})$ be the injective radius of $\mathcal{M}$, such that for any $x\in \mathcal{M}$ and $r<\mathrm{inj}(\mathcal{M})$, the exponential map $\mathrm{Exp}_{x}|_{B_r(x)}: B_r(x) \rightarrow \mathcal N_r(x)$ is a diffeomorphism when restricted on $B_r(x)$. 
\end{definition}

In general, since the geodesic equation $ \frac{d^{2} \gamma ( t)}{dt^{2}} =0$ is of second order, the exponential map does not have a closed form expression and is extremely expensive to compute. In this case, a retraction map is used to approximate it.

\begin{definition}\label{Retraction}
A (first-order) retraction on $\mathcal M$ is a differentiable map $R:\mathcal{T}\mathcal M\to\mathcal M$ such that $R|_\mathcal{M}=\mathrm{Id}$, and for each $x\in\mathcal M$, the map $R_x=R|_{\mathcal{T}_x\mathcal M}$ satisfies $DR_{x}(0) =\mathrm{Id}_{\mathcal{T}_x\mathcal M}$. 
% Moreover, $R$ is called a second-order retraction if, for every $x\in\mathcal M$ and $v\in \mathcal{T}_x\mathcal M$, the curve $t\mapsto R_x(tv)$ has zero covariant acceleration at $t=0$, i.e.,
% $\frac{d^{2} R_{x}( tv)}{dt^{2}} |_{t=0} =0$.
\end{definition}

%Obviously, the second-order retraction provides a better approximation to the exponential map. 
A typical example of retraction is the unit sphere $\mathbb{S}^{d-1}$. For any point $x\in \mathbb{S}^{d-1}$ and $v\in \mathcal{T}_{x}\mathbb{S}^{d-1}$, a retraction is $R_{x}(v) =\frac{x+v}{\|x+v\|}$, that is, the projection of $x+v$. %It can be proved that $R_{x}$ is a second-order retraction and matches the exponential map up to second order (the error is $\mathcal{O}(\|v\|_x^3)$). 
The following definition measures how close a retraction is to the exponential map, in view of Lemma \ref{lemma1:retraction}.

\begin{definition}\label{Lphi}
    Let $ R$ be a retraction on $\mathcal{M}$, $R$ is called $L_\Phi$-regular within radius $r$, if, for any $x\in \mathcal{M}$, there exists $0<\rho<\mathrm{inj}(\mathcal{M})$ such that $R_x(B_r(x))\subseteq\mathcal{N}_\rho(x)$, the function $ \Phi_x(u) = \mathrm{Exp}^{-1}_x(\mathcal{R}_x(u))$ is twice continuously differentiable for $u\in B_r(x)$, and 
        $$\sup _{t\in [ 0,1]} \|D^{2} \Phi _{x}( tu)[ u,u] \|_x\leqslant L_{\Phi }\Vert u\Vert _{x}^{2},\quad \forall x\in\mathcal{M}, u\in B_r(x). $$ 
\end{definition}

With the notion of geodesics, we can define convexity on Riemannian manifolds.

\begin{definition}
A subset $A\subseteq \mathcal{M}$ is geodesically convex if for any $x,y\in A$ and for every
minimizing geodesic segment $\gamma:[0,1]\to\mathcal{M}$ connecting $x$ and $y$, we have $ \gamma ( t) \in A$, $\forall t\in [ 0,1]$.
\end{definition}

Note that we do not require the minimizing geodesic segment connecting $x$ and $y$ to be unique. Since $(\mathcal M,g)$ is assumed to be complete, by the Hopf--Rinow theorem any two points in $\mathcal M$ can be joined by a minimizing geodesic segment; hence the whole manifold $\mathcal{M}$ is geodesically convex. For an arbitrary subset $B\subseteq \mathcal{M}$, the convex hull of $B$ is defined as the
intersection of all geodesically convex subsets $A\subseteq \mathcal M$ that contain $B$.
By construction, this convex hull is geodesically convex.

Next, we introduce the convexity of a differentiable function. Recall that in Euclidean space, a function $h$ is convex if for any $x,y$,
\begin{equation*}
h( tx+( 1-t) y) \leqslant th( x) +( 1-t) h( y) ,\quad \forall t\in [ 0,1] .
\end{equation*}
%We say $h$ is $\mu$-strongly convex if $h(\cdot)-\frac{\mu}%{2}\|\cdot\|^2$ is convex.  
For function on the Riemannian manifold, we adopt the definition of Zhang and Sra \cite{zhang2016first}.
\begin{definition}\label{def:geo_convex}
Let $A\subseteq\mathcal M$ be geodesically convex and let $f:\mathcal M\to\mathbb R$.
We say that $f$ is geodesically convex on $A$ if for any $x,y\in A$ and for every minimizing
geodesic segment $\gamma:[0,1]\to\mathcal M$ connecting $x$ and $y$,
\[
f(\gamma(t))\leqslant (1-t)f(x)+t f(y),\quad \forall t\in[0,1].
\]
% Moreover, $f$ is geodesically $\mu$-strongly convex on $A$ if for any such $x,y,\gamma$,
% \[
% f(\gamma(t))\leqslant (1-t)f(x)+t f(y)-\frac{\mu}{2}t(1-t)\,d_{\mathcal M}(x,y)^2,
% \quad \forall t\in[0,1].
% \]
\end{definition}

If we replace the exponential map $\mathrm{Exp}_{x}$ with retraction $ R_{x}$ in Definition \ref{Lphi}, we get the definition of retraction convex function. If a function is both geodesically convex and differentiable, we have
\begin{equation*}
f\left(\mathrm{Exp}_{x}( v)\right) -f( x) \geqslant \langle \nabla f( x) ,v\rangle _{x} .
\end{equation*}

In optimization, a common assumption is the Lipschitz property of gradient.
\begin{definition}\label{L-smooth}
A function $ f$ on $(\mathcal{M},g)$ is $ L$-smooth if 
\begin{equation*}
\left\Vert \Gamma _{x}^{y} \nabla f( x) -\nabla f( y)\right\Vert _{y} \leqslant L d_{\mathcal{M}}( x,y),\quad \forall x,y\in \mathcal{M} ,
\end{equation*} where  $d_{\mathcal M}$ is the geodesic distance induced by the Riemannian metric $g$ and $ \Gamma _{x}^{y} :\mathcal{T}_{x}\mathcal{M}\rightarrow \mathcal{T}_{y}\mathcal{M}$ is the parallel transport along any minimizing geodesic connecting $ x$ and $ y$.
\end{definition}
\begin{remark}
    We remark that this definition is mainly about the local property of function $f$. Since parallel transport is an isometry, the left hand side of the inequality is $\left\Vert \Gamma _{x}^{y} \nabla f( x) -\nabla f( y)\right\Vert _{y} \leqslant \left\Vert \Gamma _{x}^{y} \nabla f( x) \right\Vert _{y} +\left\Vert \nabla f( y)\right\Vert _{y} =\left\Vert \nabla f( x)\right\Vert _{x}+\left\Vert \nabla f( y)\right\Vert _{y}\leqslant 2G$ by Assumption \ref{asp:bounded_gradient}, thus when $ d_{\mathcal{M}}( x,y)\geqslant \frac{2G}{L}$ the definition is satisfied trivially. And also when $ x$ and $ y$ are close enough, for example when $ d_\mathcal{M}(x,y)<\mathrm{inj}(\mathcal{M})$, the minimizing geodesic connecting them is unique.
\end{remark}

%By the completeness of $ \mathcal{M}$, there exists minimizing geodesic connecting $ x$ and $ y$ but the geodesic is not necessarily unique. But this is not a big deal because 

%at least when $ x$ and $ y$ are close enough, for example when $ x$ is within the injective radius of $ y$, there is a unique minimizing geodesic connecting $ x$ and $ y$. 

We conclude by introducing some notation used throughout the paper. 
%For $x\in\mathcal M$ and $r>0$, we denote by $\mathcal N_r(x)=\{y\in\mathcal M:\ d_{\mathcal M}(x,y)\leqslant r\}$ the (closed) geodesic ball of radius $r$ centered at $x$. 
Given a matrix $A=[a_{ij}]$ and index sets $\mathcal I,\mathcal J$, we denote by $A(\mathcal I,\mathcal J)=[a_{ij}]_{i\in\mathcal I,\,j\in\mathcal J}$
the submatrix of $A$ with row indices $\mathcal I$ and column indices $\mathcal J$. We use the shorthand $[n]=\{1,2,\dots,n\}$. Finally, $\nabla_\mathcal{M}$ means taking gradient on manifold $ \mathcal{M}$.

% \subsection{Mirror Descent}
% We briefly review the classical MD algorithm in Euclidean space. Consider the classical optimization problem $\min_{x\in \Omega } f( x)$. Given a strongly convex function $ \psi $ (called the potential function in the literature), we define the Bregman distance generated by $ \psi $ to be 
% \begin{equation*}
% D_{\psi }( x,y) =\psi ( x) -\psi ( y) -\langle \nabla \psi ( y) ,x-y\rangle .
% \end{equation*}The update rule of MD is 
% \begin{equation*}
% x_{t+1} =\arg\min_{x\in \Omega } \langle \eta _{t} \nabla f( x_{t}) ,x-x_{t} \rangle +D_{\psi }( x,x_{t}) .
% \end{equation*}The Bregman distance term can be interpreted as a regularization term. If $ \Omega =\mathbb{R}^{d}$, by the optimality condition, 
% \begin{equation*}
% \nabla \psi ( x_{t+1}) =\nabla \psi ( x_{t}) -\eta _{t} f( x_{t}) .
% \end{equation*}
%  By choosing different potential functions $ \psi $, one can make use of the geometry of the feasible region. For example, if the feasible region is a simplex,
% $$ \Omega =\left\{x:\sum _{i=1}^{d} x_{i} =1, x_i\geqslant 0, \forall i \in [d] \right\},$$ 
% $ \psi ( x) =x\log x$ may be a good choice for large-scaled optimization 
%  \cite{nemirovski2009robust}. In this case, the Bregman distance $ D_{\psi }$ becomes the Kullback-Leibler (KL) divergence. In addition, taking $ \Omega =\mathbb{R}^{d}$ and $ \psi ( x) =\frac{1}{2}\Vert x\Vert_2 ^{2}$ gives the commonly used Gradient Descent (GD) algorithm. 

\section{Riemannian Mirror Descent: Algorithms and Convergence Theorems}\label{sec:algo}

In this section, we formulate Riemannian Mirror Descent (RMD) and stochastic RMD on manifolds. We establish non-asymptotic convergence guarantees for both algorithms. We then present several concrete examples arising from our RMD framework. In particular, we obtain Stochastic Curvilinear Gradient Descent (SCGD) as a special case of particular interest.

\subsection{Motivation and Algorithm Framework}

In this subsection, the ambient space is $\mathbb{R}^d$ equipped with its standard coordinates, and $\nabla$ and $\nabla^2$ denote the usual Euclidean gradient and Hessian, respectively.  Our main motivation comes from the observation of Raskutti and Mukherjee \cite{raskutti2015information}. They consider a different Riemannian structure on $\mathbb{R}^d$.  Given a strongly convex smooth function $\psi$, for each point $x \in \mathbb{R}^{d}$, we define 
\begin{equation*}
g_{x}^{\psi}(v,u) = v^{\mathsf{T}} \nabla^{2} \psi(x)u, \quad \forall u,v \in \mathbb{R}^{d}.
\end{equation*}
Notice that $g_{x}^{\psi}(\cdot,\cdot)$ is an inner product on $\mathbb{R}^{d}$ since $\nabla^{2} \psi(x)$ is positive definite. Under the standard
identification $\mathcal{T}_x\mathbb{R}^d\simeq\mathbb{R}^d$, the matrix representation of $g_{x}^{\psi}(\cdot,\cdot)$ is
$g^{\psi}(x)=\nabla^2\psi(x)$. We denote by $\mathcal{M}_{\psi} = (\mathbb{R}^{d}, g^{\psi})$ the Riemannian manifold obtained by equipping $\mathbb{R}^{d}$ with the Riemannian metric $g^{\psi}$. 
Let $\psi^*$ be the conjugate of $\psi$, defined by $\psi^{*}(y) = \sup_{x} \left\{\langle y, x\rangle - \psi(x)\right\}$. Then $\psi^{*}$ is also strongly convex and smooth. Let $\mathcal{M}_{\psi^{*}} = (\mathbb{R}^{d}, g^{\psi^{*}})$ be the manifold in the dual variable $y$. Then the map $x\mapsto \nabla\psi(x)$ is from $\mathcal{M}_{\psi}$ to $\mathcal{M}_{\psi^{*}}$, and $ \nabla \psi^{*}(\nabla \psi(x))=x$, $\nabla^{2} \psi(x) \nabla^{2} \psi^*(\nabla\psi(x))=\mathrm{Id}_{\mathbb{R}^d}$.

The update in mirror descent (MD) can be rewritten as
\begin{align}
y_{t+1} & = \nabla \psi(x_{t}) - \eta_{t} \nabla f(x_{t}), \label{eq:MD1} \\
x_{t+1} & = \nabla \psi^{*}(y_{t+1}).
\end{align}
%where we use $(\nabla \psi)^{-1} = \nabla \psi^{*}$. 
In the first step, $\nabla\psi$ maps $x_t\in\mathcal M_\psi$ to $\nabla\psi(x_t)\in\mathcal M_{\psi^*}$,
and the update subtracts $\eta_t\nabla f(x_t)$ in the dual variable.  Under our setting, $\nabla\psi$ is invertible
with inverse $(\nabla\psi)^{-1}=\nabla\psi^*$. The invertible map $\nabla \psi$ pulls back $f$ on $\mathcal{M}_{\psi}$ to $\tilde{f}=f\circ(\nabla \psi)^{-1}=f\circ\nabla\psi^{*}$ on $\mathcal{M}_{\psi^{*}}$.
We can calculate the Riemannian gradient of $\tilde{f}$ at $\nabla \psi(x_{t})$ on $\mathcal{M}_{\psi^{*}}$. Note that $y_{t}=\nabla \psi(x_{t})$, we have
\begin{align*}
\nabla_{\mathcal{M}_{\psi^{*}}} \tilde{f}(y_{t}) 
&= (g^{\psi^{*}}(y_{t}))^{-1} \nabla \tilde{f} (y_{t})\\
&= (\nabla^{2} \psi^{*}(y_{t}))^{-1} \cdot \nabla\nabla\psi^{*}(y_{t}) \cdot \nabla f(\nabla\psi^{*}(y_{t})) \\
&= (\nabla^{2} \psi^{*}(y_{t}))^{-1} \cdot \nabla^{2} \psi^{*}(y_{t}) \cdot \nabla f(x_{t}) = \nabla f(x_{t}),
\end{align*}
where the first line is the coordinate expression for Riemannian gradient and in the second line we use the chain rule. This computation means that $\nabla f(x_{t})$ is equivalent to the gradient of $\tilde{f}$ in $\mathcal{M}_{\psi^*}$ at $y_t$. As a result, the first step is a gradient step $y_{t+1}=y_{t}-\eta_{t}\nabla_{\mathcal{M}_{\psi^{*}}} \tilde{f}(y_{t})$ in the space $\mathcal{M}_{\psi^*}$. In the second step, the inverse of $\nabla \psi$, i.e., $\nabla \psi^{*}$ maps $y_{t+1}$ back to the primal manifold $\mathcal{M}_{\psi}$.

The MD update process can be viewed as gradient descent via reparameterization. More precisely, by changing of variables $y=\nabla\psi(x)$, the algorithm maps the primal iterate $x_t\in\mathcal M_\psi$ to the dual coordinate $y_t\in\mathcal M_{\psi^*}$, performs gradient descent in the dual space $\mathcal{M}_{\psi^{*}}$, and then maps back to $\mathcal M_\psi$ via $\nabla\psi^*$. Motivated by this idea, we propose the Riemannian Mirror Descent (RMD) algorithm (Algorithm~\ref{algo:md_rie}) as an extension of Euclidean MD to general Riemannian manifolds.

\begin{algorithm}[!htb]
    \caption{Riemannian Mirror Descent (RMD)}
    \label{algo:md_rie}
    \begin{algorithmic}
    \STATE \textbf{Input:} {Iteration number $ T $, Initial point $ x_0 $, step size$ \{\eta_t\}_{t=0}^{T-1} $, radius$ \{r_t\}_{t=0}^{T-1} $ } \\
    \FOR{$t = 0$ {to} $T-1$:}
      \STATE Construct local diffeomorphism: $\varphi_t: \hat{\mathcal{N}}_{r_t}(x_t) \rightarrow \mathcal{M}_t $, and retraction  $ R^t_{\varphi_t(x_t)} $ on $\mathcal{M}_t  $ \\
      \STATE $y_{t+1} = R^t_{\varphi_t(x_t)}(-\eta_t D\varphi_t(x_t)[\nabla f(x_t)]) $ \hfill{Dual update}\\
      \STATE $ x_{t+1} = \varphi_t^{-1}(y_{t+1})$ \hfill{Back to primal space}\\
    \ENDFOR
    \RETURN $ x_T $
    \end{algorithmic}
\end{algorithm}

In Algorithm~\ref{algo:md_rie}, we replace the global map $x\mapsto \nabla\psi(x)$ by a sequence of local diffeomorphisms $\varphi_t$. 
%For simplicity, we assume that $\mathcal M$ is boundaryless and that $\eta_t \nabla f(x_t) \in \mathcal{N}_{r_t}(x_t)$ in all iterations. 
At each iteration, we choose a local diffeomorphism $\varphi_t:\hat{\mathcal{N}}_{r_t}(x_t)\to \mathcal M_t$ onto a reparameterized manifold $\mathcal M_t$. 
Here $\hat{\mathcal{N}}_{r}(x)=\{y\in\mathcal M:d_{\mathcal M}(x,y)< r\}$ is the open geodesic ball. For simplicity, when $\mathrm{inj}(\mathcal{M})<+\infty$, we assume $r_t< \mathrm{inj}(\mathcal{M})$ so that $\hat{\mathcal{N}}_{r_t}(x)$ is diffeomorphic to open ball $\{v\in \mathcal{T}_x \mathcal{M}:\|v\|<r\}$, and so is $\mathcal{M}_t$. When $\mathrm{inj}(\mathcal{M})=+\infty$, $r_t$ is free to be finite or infinite.
We then perform a gradient step on $\mathcal M_t$ via retraction $R^t$ and map the iterate back via $\varphi_t^{-1}$. Updating $y_{t+1}$ in Algorithm~\ref{algo:md_rie}  corresponds to~\eqref{eq:MD1}. In essence, RMD amounts to carrying out gradient descent after a local reparameterization. One benefit is that the retraction map can be easier to compute in the reparameterized space.

Next, we present the convergence theorem for RMD in Algorithm \ref{algo:md_rie}. To establish these results, we introduce the following assumptions.

\begin{assumption}\label{asp:Phi}
%        Let $ \mathcal{R}_x: \mathcal{T}_x\mathcal{M}\rightarrow \mathcal{M}$ be a retraction and define $ \Phi_x(u) = \text{Exp}^{-1}_x(\mathcal{R}_x(u)), \forall x\in\mathcal{M}$. Assume that for each $x\in\mathcal{M}$, $\Phi_x$ is twice continuously differentiable in $u \in \mathcal{T}_x\mathcal{M}$. Moreover, there exists a constant $ L_\Phi >0$ such that 
%        $$\sup _{t\in [ 0,1]} \|D^{2} \Phi _{x}( tu)[ u,u] \|_x\leqslant L_{\Phi }\Vert u\Vert _{x}^{2},\quad \forall x\in\mathcal{M}, u\in\mathcal{T}_x\mathcal{M}. $$ 
The map $ \hat{R}_{x}^t( u) =\varphi _{t}^{-1}( R_{\varphi _{t}( x)}^t( D\varphi _{t}( x)[ u])) $ is a retraction and it is $L_\Phi$-regular within radius $r$ (see Definition \ref{Retraction} and Definition \ref{Lphi}).
    \end{assumption}

   \begin{assumption}\label{asp:A}
        There exists a compact geodesically convex set $ A\subset \mathcal{M} $ such that all iteration points $x_t\in A$ for all $t=1,\dots,T$, and a minimizer $x^*$ satisfies $x^*\in A$.
    \end{assumption}

\begin{assumption}[Smoothness]\label{asp:smooth}
        The objective function $f$ is $L_f$-smooth (See Definition \ref{L-smooth}). 
    \end{assumption}
    \begin{assumption}[Bounded Gradient]\label{asp:bounded_gradient}
    The gradient is bounded, i.e., there exists a constant $G >0$ such that $\| \nabla f(x)\|_x \leqslant G$ for all $x\in\mathcal{M}$. 
    \end{assumption}

    \begin{remark}
We remark that in Assumption \ref{asp:A}, it is sufficient to take $A$ to be the geodesic convex hull of $\{x_1,\ldots,x_T,x^*\}$, provided that this set $A$ is bounded. Moreover, if the initialization $x_1$ lies in a neighborhood of $x^*$, it is natural to assume that $f$ is geodesically convex on $A$. We introduce a subset $A$ rather than working on the whole manifold $\mathcal M$ in order to avoid global geometric obstructions to geodesic convexity. For instance, although Hadamard manifolds admit geodesically convex (even strongly convex) functions, imposing geodesic $L$-smoothness simultaneously can be subtle; see Proposition 28 in \cite{criscitiello2022negative}. A more direct example is that on a closed (i.e., compact and boundaryless) and connected Riemannian manifold there is no nonconstant globally geodesically convex function. 

If it is difficult to obtain such an initialization and global $L$-smoothness and convexity on $A$ cannot be ensured, the ``nonconvex'' part of the convergence theorems below still applies. The ``convex'' part can be interpreted as a tail-rate (local) analysis: view $A$ as a neighborhood of $x^*$ on which $f$ is geodesically convex, and assume that the tail iterates $\{x_n,x_{n+1},\ldots,x_T,x^*\}$ are contained in $A$ for some $n\in \mathbb{N}$.
    \end{remark}

We are now in a position to state the convergence theorem.

\begin{theorem}[Convergence of RMD]\label{thm:rmd}
For RMD updates in Algorithm \ref{algo:md_rie}, assume that Assumptions~\ref{asp:Phi}--\ref{asp:bounded_gradient} hold. 
Let $\eta_t\equiv \eta$ be a constant step size satisfying
\[
0<\eta<\min\left\{\frac{1}{L_{\Phi}G/2+L_f+L_fL_{\Phi}^2},\ \frac{2}{G},\ \frac{r}{G}\right\}.
\] 
Then the following statements hold:
\begin{enumerate}
\item (Nonconvex case) 
\begin{equation*} 
% \min_{t=1,\cdots,T} \| \nabla f(x_{t} )\| _{x_{t}}^{2} \leqslant \mathcal{O}\left(\frac{1}{T}\right). 
\frac{1}{T}\sum _{t=1}^{T} \| \nabla f(x_{t} )\| _{x_{t}}^{2} \leqslant \mathcal{O}\left(\frac{1}{T}\right).
\end{equation*}

\item (Geodesically convex case) If, in addition, $f$ is geodesically convex on $A\subset \mathcal{M}$ (from Assumption~\ref{asp:A}), then 
\begin{equation*} f( x_{T}) -f\left( x^{*}\right) \leqslant \mathcal{O}\left(\frac{1}{T}\right) . 
\end{equation*}
\end{enumerate}
\end{theorem}
    
A detailed proof is given in Section \ref{sec:thm}.

\subsection{Examples of Riemannian Mirror Descent}

Algorithm~\ref{algo:md_rie} establishes an abstract framework for RMD. By choosing different mirror maps $\varphi_t$ and manifolds $\mathcal{M}_t$, one can recover a variety of algorithms as special cases of this framework. In what follows, we present several concrete examples to illustrate our framework.

\begin{example}[Euclidean Mirror Descent]
    Given a strongly convex second-order differentiable function $ \psi $, take $ \mathcal{M} = \mathcal{M}_{\psi} = (\mathbb{R}^d,g^\psi) $, $ \mathcal{M}_t \equiv \mathcal{M}_{\psi^*}= (\mathbb{R}^d,g^{\psi^*}) $, where $g^\psi$ is defined in Section 3.1, $ R_x(v) = x + v $ and $ \varphi_t(x) \equiv \nabla\psi(x)$, Algorithm \ref{algo:md_rie} recovers the MD algorithm in Euclidean space.   
\end{example}
In this example, the mirror map is in fact an isometry. 
Fix $x\in \mathcal{M}_{\psi }$ and $u,v\in \mathcal{T}_{x}\mathcal{M}_{\psi }$. 
Let $y=\nabla\psi(x)$ and denote 
$\hat u = D(\nabla\psi)(x)[u]=\nabla^2\psi(x)\,u$ and $\hat v=\nabla^2\psi(x)\,v$.
Then
\[
g^{\psi^*}_{y}(\hat u,\hat v)
= \hat u^{\mathsf T}\nabla^2\psi^*(y)\hat v
= u^{\mathsf T}\nabla^2\psi(x)\,\nabla^2\psi^*(\nabla\psi(x))\,\nabla^2\psi(x)\,v
= u^{\mathsf T}\nabla^2\psi(x)\,v
= g^{\psi}_{x}(u,v),
\] where we used the identity $\nabla^2\psi^*(\nabla\psi(x))=(\nabla^2\psi(x))^{-1}$.
Therefore, using the conjugacy relation between $\psi$ and $\psi^*$, the mirror map $\nabla\psi$ in classical MD is a Riemannian isometry between $\mathcal M_{\psi}$ and $\mathcal M_{\psi^*}$.
\par 

On a general Riemannian manifold, one cannot in general expect to construct a globally defined Riemannian isometry to serve as a mirror map. However, the tangent space provides a natural local parameterization of the manifold. By using the exponential map as a mirror map, Algorithm~\ref{algo:md_rie} recovers the Geodesic Gradient Descent algorithm.
\begin{example}[Geodesic Gradient Descent]
Let $\mathcal M$ be a Riemannian manifold. In Algorithm~\ref{algo:md_rie}, fix $x_t\in\mathcal M$ and choose $r<\mathrm{inj}(\mathcal{M})$ so that $\mathrm{Exp}_{x_t}$ is a diffeomorphism on the geodesic ball $B_{r}(x_t)$. Define
\[
\varphi_t(x)=\mathrm{Exp}_{x_t}^{-1}(x),\quad x\in \hat{\mathcal{N}}_{r}(x_t),\quad R_x(v)=x+v,\quad 
\mathcal M_t =\{v\in \mathcal{T}_{x_t} \mathcal{M}:\|v\|<r\}\subseteq \mathcal{T}_{x_t}\mathcal M.
\]
Then the resulting update coincides with Geodesic Gradient Descent \cite{zhang2016first}.
\end{example}

Indeed, $\mathrm{Exp}_{x_t}(0)=x_t$ and $D\mathrm{Exp}_{x_t}(0)=\mathrm{Id}_{\mathcal{T}_{x_t}\mathcal M}$. Since $\varphi_t(x_t)=\mathrm{Exp}_{x_t}^{-1}(x_t)=0$, we obtain
\[
y_{t+1}
=R_{\varphi_t(x_t)} \left(-\eta_t\,D\varphi_t(x_t)[\nabla f(x_t)]\right)
=-\eta_t \nabla f(x_t),
\]
and then \[
x_{t+1}=\varphi_t^{-1}(y_{t+1})
=\mathrm{Exp}_{x_t} \left(-\eta_t \nabla f(x_t)\right),
\]
which is exactly Geodesic Gradient Descent.

Next, we discuss how to apply our framework to optimization over the Stiefel manifold.

\begin{definition}
    For $1\le p\le n$, the Stiefel manifold is
$$ \mathrm{St}(n,p)=\left\{X\in \mathbb{R}^{n\times p} :X^{\mathsf{T}} X=I_{p}\right\}. $$ Its tangent space at $X\in \mathrm{St}(n,p)$ is
\[
\mathcal{T}_X\mathrm{St}(n,p)=\{Z\in\mathbb R^{n\times p}:X^{\mathsf T}Z+Z^{\mathsf T}X=\mathbf{0}\}.
\]
The canonical metric is defined by, for all $A,B\in \mathcal{T}_X\mathrm{St}(n,p)$,
\[
\langle A,B\rangle_X=\mathrm{tr}\left(A^{\mathsf T}\Big(I_n-\tfrac12 XX^{\mathsf T}\Big)B\right).
\]
\end{definition} From this definition, the Stiefel manifold with $p=n$ is
$\mathrm{St}(n,n)=\{X\in\mathbb{R}^{n\times n}:X^{\mathsf T}X=I_n\}.$ Then $I_n\in \mathrm{St}(n,n)$ and
\[
\mathcal T_{I_n}\mathrm{St}(n,n)=\mathrm{skew}(n),
\]
where $\mathrm{skew}(n)=\{W\in\mathbb R^{n\times n}: W+W^{\mathsf T}=0\}$ denotes the set of $n\times n$ skew-symmetric matrices. Therefore, the Cayley transform provides a natural local reparameterization of $\mathrm{St}(n,n)$ near the identity. 
For any $W\in \mathrm{skew}(n)$, the Cayley transform is defined by
\[
C(W)=(I_n-W)^{-1}(I_n+W)\in \mathrm{St}(n,n).
\]Then $C$ is a diffeomorphism between $\mathrm{skew}(n)$ and the open subset
$\{X\in \mathrm{St}(n,n): X+I_n \text{ is invertible}\}$.
On this subset, its inverse is
\[
\varphi(X)=C^{-1}(X)=(X-I_n)(X+I_n)^{-1}.
\]
Viewing $\varphi$ as a smooth map into $\mathrm{skew}(n)$, we have
\[
D\varphi(I_n)[W]=\tfrac12 W,\qquad \forall\, W\in \mathrm{skew}(n).
\]Thus, when $x_t=I_n$ and using the Euclidean retraction on $\mathrm{skew}(n)$, the update in dual variable $y$ is
\[
y_{t+1}=-\eta_t\,D\varphi(I_n)[\nabla f(I_n)]
= -\tfrac{\eta_t}{2}\nabla f(I_n)\in \mathrm{skew}(n),
\]
and mapping back yields
\[
x_{t+1}= \varphi^{-1}(y_{t+1})
= C(y_{t+1})
= (I_n-y_{t+1})^{-1}(I_n+y_{t+1})\in \mathrm{St}(n,n).
\]  For a general base point $x_t=X_0\in \mathrm{St}(n,n)$, since $\mathrm{St}(n,n)$ is a homogeneous manifold, we use the translation
\[
O_{X_0}:\mathrm{St}(n,n)\to\mathrm{St}(n,n),\qquad O_{X_0}(X)=X X_0^{\mathsf T},
\]
which satisfies $O_{X_0}(X_0)=I_n$. 
On a neighborhood of $X_0$ where $O_{X_0}(X)+I_n$ is invertible, we define
\[
\varphi(X)=C^{-1} \big(O_{X_0}(X)\big)
=\big(XX_0^{\mathsf T}-I_n\big)\big(XX_0^{\mathsf T}+I_n\big)^{-1}.
\]
This provides a local reparameterization of $\mathrm{St}(n,n)$ around $X_0$ with $\varphi(X_0)=\mathbf{0}$.

\begin{example}[Curvilinear Gradient Descent on $\mathrm{St}(n,n)$] In Algorithm~\ref{algo:md_rie}, let $\mathcal M_t = \mathcal{T}_{I_n}\mathrm{St}(n,n)=\mathrm{skew}(n)$. 
Define $\varphi_t$ at $X_t$ by
\[
\varphi_t(X)=\big(XX_t^{\mathsf T}-I_n\big)\big(XX_t^{\mathsf T}+I_n\big)^{-1},
\]
which is well-defined whenever $XX_t^{\mathsf T}+I_n$ is invertible, and use Euclidean retraction
$R_X(V)=X+V$ on the vector space $\mathrm{skew}(n)$.
\end{example}

Since $\mathcal{T}_{X_t}\mathrm{St}(n,n)=\{W X_t: W\in \mathrm{skew}(n)\}$, the Riemannian gradient can be written as
$\nabla f(X_t)=W_tX_t$ for some $W_t\in \mathrm{skew}(n)$.
The RMD update then in Algorithm~\ref{algo:md_rie} admits the explicit form
\begin{equation}\label{eq:rmd_sti}
X_{t+1}=\big(I_n+\tfrac{\eta_t}{2}W_t\big)^{-1}\big(I_n-\tfrac{\eta_t}{2}W_t\big)X_t,
\end{equation}
which coincides with the Curvilinear Gradient Descent (CGD) update in~\cite{wen2013feasible}. For the general Stiefel manifold $\mathrm{St}(p,n)$ with $p<n$, we may extend any
$X\in \mathrm{St}(n,p)$ to an orthogonal matrix by choosing an orthonormal
complement $X_{\perp}\in\mathbb{R}^{n\times(n-p)}$ such that
$
\bar X = [\,X,\;X_{\perp}\,]\in \mathrm{St}(n,n).
$
Applying the same construction to $\bar X$ and then taking its first $p$ columns
yields the iteration \eqref{eq:rmd_sti} on $\mathrm{St}(n,p)$.

\subsection{Stochastic Riemannian Mirror Descent} 

In this subsection, we introduce a stochastic variant of Riemannian Mirror Descent (RMD). Suppose that, instead of the exact Riemannian gradient $\nabla f(x)$, we have access to an unbiased stochastic estimator $\nabla f(x,\xi)$ satisfying
\[
\mathbb{E}\left[\nabla f(x,\xi)\right]=\nabla f(x),
\]
where $\xi$ denotes the randomness in the oracle. In this setting, we can run Algorithm \ref{algo:md_rie} by replacing the exact gradient with its stochastic estimate. In particular, the dual update becomes
\begin{equation}\label{dualupdate_in_RMBD}
y_{t+1}
= R^{t}{\varphi_t(x_t)}\left(-\eta_t D\varphi_t(x_t)\big[\nabla f(x_t,\xi_t)\big]\right),
\end{equation}
where $\{\xi_t\}_{t=1}^T$ are independent sampling noises.
We refer to this algorithm as Stochastic Riemannian Mirror Descent (SRMD). Analogous to Theorem~\ref{thm:rmd}, we establish the following convergence theorem for SRMD.

%\begin{assumption}\label{asp:sto-bounded}
%    For all $x\in \mathcal{M}$, $\left\| \nabla f(x,\xi) \right\|_{x}^2 \leqslant G$ with the same $G$ as in Assumption \ref{asp:bounded_gradient}.
%\end{assumption}

\begin{theorem}[Convergence of SRMD]\label{thm:SRMD}
Assume that Assumptions~\ref{asp:Phi}--\ref{asp:bounded_gradient} hold. If the dual update in Algorithm \ref{algo:md_rie} is given by \eqref{dualupdate_in_RMBD}, where $\{\xi_t\}_{t=1}^T$ are i.i.d. random variables satisfying $\mathbb{E}\left[\nabla f(x_t,\xi_t)|x_t\right]=\nabla f(x_t)$, $\left\| \nabla f(x_t,\xi_t) \right\|_{x_t} \leqslant G$ and $\sigma^2= \sup_{t}\mathbb{E} [\left\| \nabla f(x_t,\xi_t) -\nabla f(x_t) \right\|^2 |x_t] <\infty$. 
Let $\eta_t\equiv \eta$ be a constant step size satisfying
\[
0<\eta<\min\left\{\frac{1}{L_{\Phi}G+2L_f+2L_fL_{\Phi}^2},\ \frac{2}{G},\ \frac{r}{G}\right\}.
\] 
Then the following statements hold:
\begin{enumerate}
\item (Nonconvex case) 
\begin{equation*} 
\frac{1}{T}\sum _{t=1}^{T}\mathbb{E} \| \nabla f(x_{t} )\| _{x_{t}}^{2} \leqslant \mathcal{O}\left(\frac{1}{\eta T} +\eta \sigma ^{2}\right).
\end{equation*}

\item (Geodesically convex case) If, in addition, $f$ is geodesically convex on $A\subset \mathcal{M}$ (from Assumption~\ref{asp:A}), then 
\begin{equation*} 
\mathbb{E} f( x_{T}) -f\left( x^{*}\right) \leqslant \mathcal{O}\left( \eta ^{2}T \sigma ^{2} +\frac{1}{\eta T}\right).
\end{equation*}
\end{enumerate}

In particular, take $ \eta \propto \frac{1}{T^{1/2}}$ and $\eta \propto \frac{1}{T^{2/3}}$, respectively. Then we have
$$\frac{1}{T}\sum _{t=1}^{T}\mathbb{E} \| \nabla f(x_{t} )\| _{x_{t}}^{2} \leqslant \mathcal{O}\left(\frac{1}{\sqrt{T}}\right) $$ in nonconvex case and
        $$  \mathbb{E} f( x_{T}) -f\left( x^{*}\right) \leqslant \mathcal{O}\left( \frac{1}{T^{1/3}}\right)$$
        in the geodesically convex case.
    \end{theorem}
A detailed proof is given in Section \ref{sec:thm}.

As an illustration, we specialize SRMD to optimization over the Stiefel manifold
$\mathrm{St}(n,p)$ and obtain a randomized algorithm suitable for the regime
where $p$ is large. Wen and Yin~\cite{wen2013feasible} use the Sherman--Morrison--Woodbury (SMW) formula to accelerate the Cayley-type update
\eqref{eq:rmd_sti} when $2p \ll n$, exploiting the fact that the associated
skew-symmetric matrix has rank at most $2p$. When $2p \geq  n$, this low-rank
reduction no longer yields a smaller system to invert (and hence provides little
or no computational advantage); in this regime, one typically resorts to direct
solvers or other implementations; see~\cite{wen2013feasible}. Motivated by SRMD,
we propose the Stochastic Curvilinear Gradient Descent (SCGD) algorithm, which
is amenable to parallel implementation and is effective when $p$ is large.

We need the following lemma (Lemma~1 in~\cite{wen2013feasible}) to express the
tangent-space gradient in terms of a skew-symmetric matrix.
\begin{lemma}[Lemma~1 in~\cite{wen2013feasible}]
Let $f:\mathbb{R}^{n\times p}\to\mathbb{R}$ be differentiable and let
$X\in \mathrm{St}(n,p)$.  Then 
$$ \nabla_{\mathrm{St}(n,p)} f(X) = WX, \quad W= \nabla_{\mathbb{R}^{n\times p}} f(X)X^\top - X(\nabla_{\mathbb{R}^{n\times p}} f(X))^\top \in \mathbb{R}^{n\times n}.$$ 
\end{lemma}

The CGD update requires applying $(I_n+\eta_t W_t/2)^{-1}$. When $2p\ll n$, $W_t$ has rank at most $2p$ and the SMW formula can substantially reduce the computational cost. When $2p\geq n$, this reduction no longer yields a smaller
system to invert, and computing $(I_n+\eta_t W_t/2)^{-1}$ can be expensive.

\begin{algorithm}[!htb]
        \caption{Stochastic Curvilinear Gradient Descent (SCGD)}
        \label{algo:scd}
        \begin{algorithmic}
        \STATE \textbf{Input:} {Iteration number $ T $, Initial point $ X_0 $, step size $ \{\eta_t\}_{t=0}^{T-1} $, parameter $K$}\\
        \FOR{$t = 0$ to $T-1$}
          \STATE Randomly divide $ [n] $ into $ K $ sets $ \mathcal{I}_1,\cdots,\mathcal{I}_K $ evenly
          \FOR{$ k = 1 $ to $ K $ }
            \STATE $ W_k = \nabla_{\mathbb{R}^{n\times p}} f(X_t)(\mathcal{I}_k,:) (X_t(\mathcal{I}_k,:))^\T - X_t(\mathcal{I}_k,:)(\nabla_{\mathbb{R}^{n\times p}} f(X_t)(\mathcal{I}_k,:))^\T $ \\
            \STATE $ X_{t+1}(\mathcal{I}_k,:) = (I_n+\eta_t/2 W_k)(I_n-\eta_t W_{k} /2)^{-1}X_{t}(\mathcal{I}_k,:) $ \\       
           \ENDFOR
        \ENDFOR
        \RETURN $ X_T $
        \end{algorithmic}
    \end{algorithm}

To address the large-$p$ regime, we construct a block-diagonal approximation by
randomly partitioning $[n]$ into $K$ equal-sized subsets
$\mathcal{I}_1,\ldots,\mathcal{I}_K$ (with $n=mK$) and defining $\hat W$ by
keeping only the within-block entries:
\[
\hat w_{ij}=
\begin{cases}
w_{ij}, & i,j\in \mathcal{I}_k \text{ for some } k,\\
0, & \text{otherwise}.
\end{cases}
\]For example, if $ \mathcal{I}_k = \{(k-1)m+1,\cdots,km\} $, then 
    $$ \hat{W} =\begin{pmatrix}
        W_{1} & O & O & O\\
        O & W_{2} & O & O\\
        O & O & \cdots  & O\\
        O & O & O & W_{K}
        \end{pmatrix} .$$
For any $i\neq j$, we have
$\mathbb{P}(i,j\text{ belong to the same block})=\frac{m-1}{n-1}$, and hence
\[
\mathbb{E}[\hat{W}] = \frac{n/K-1}{n-1} W. 
\] 
Therefore, $\frac{n-1}{n/K-1}\hat W$ is an unbiased estimator of $W$. Since
$\hat W$ is block diagonal, applying $(I+\eta_t \hat W/2)^{-1}$ reduces to
inverting $K$ independent $m\times m$ systems
$(I+\eta_t \hat W(\mathcal{I}_k,\mathcal{I}_k)/2)^{-1}$, which can be readily
parallelized. This enables the algorithm to handle large-scale problems; see
Algorithm~\ref{algo:scd} for details.

\section{Convergence Analysis}\label{sec:thm}
This section is devoted to proving Theorems~\ref{thm:rmd} and ~\ref{thm:SRMD}. To this end, we first establish the following lemma on the difference between retraction and the exponential map. Let
$\Phi_x(u) = \mathrm{Exp}_x^{-1}(R_x(u))$ be well-defined for a retraction $R_x$.

\begin{lemma}\label{lemma1:retraction}
Given a retraction $ R_x $ that is $L_\Phi$-regular within radius $r$. Then for all $u\in B_r(x)$,
\begin{equation*}
\Vert \Phi _{x}( u) -u\Vert _{x} \leqslant \frac{L_{\Phi }}{2}\Vert u\Vert _{x}^{2} .
\end{equation*}
\end{lemma}

\begin{proof}
For $u\in B_r(x)$, consider the Taylor expansion of $ g( t) =\Phi _{x}( tu)$ at $ t=0$,
    \begin{align*}
    g( 1) =\Phi _{x}( u) & =\Phi _{x}( 0) +D\Phi _{x}( 0)[ u] +\int _{0}^{1}( 1-t) D^{2} \Phi ( tu)[ u,u] dt\\
     & =u+\int _{0}^{1}( 1-t) D^{2} \Phi ( tu)[ u,u] dt,
    \end{align*}
    where we use $R_{x}(0)=x$ and $D\Phi_{x}(0)=\mathrm{Id}$.

    By assumption, there holds
    \begin{equation*}
    \Vert \Phi _{x}( u) -u\Vert _{x} \leqslant L_{\Phi }\Vert u\Vert _{x}^{2}\int _{0}^{1} (1-t)dt=\frac{L_{\Phi }}{2}\Vert u\Vert _{x}^{2}, 
    \end{equation*} which implies the desired result.
\end{proof}

Using Lemma \ref{lemma1:retraction}, we can analyze the one-step progress of RMD. Recall that $$ \hat{R}_{x}^t( u) =\varphi _{t}^{-1}( R_{\varphi _{t}( x)}^t( D\varphi _{t}( x)[u])).$$ Then $D\hat{R}_{x}^t( 0) =( D\varphi _{t}( x))^{-1} D\varphi _{t}( x) =\mathrm{Id}$, which implies that $ \hat{R}_x^t $ is a retraction.

\begin{lemma}\label{lemma2:one_step_progress}
   Suppose that $ \hat{R}_{x}^t$ satisfies Assumption \ref{asp:Phi} and $f$ satisfies Assumptions \ref{asp:A}-\ref{asp:bounded_gradient}, then 
    \begin{equation*}
    f( x_{t+1}) -f( x_{t}) \leqslant -\left( \eta _{t} -L_{\Phi } G\eta _{t}^{2} /2-\eta _{t}^{2} L_{f}\right)\Vert \nabla f( x_{t})\Vert _{x_{t}}^{2} +\frac{L_{f} L_{\Phi }^{2} \eta _{t}^{4}}{4}\Vert \nabla f( x_{t})\Vert _{x_{t}}^{4} .
    \end{equation*}
\end{lemma}
\begin{proof}
    Let $ d_{t} =-\mathrm{Exp}^{-1}_{x_t}(\hat{R}^t_{x_{t}}( -\eta _{t} \nabla f( x_{t})))$, by Assumption \ref{asp:smooth},
    \begin{align}
    f( x_{t+1}) -f( x_{t}) & \leqslant -\langle \nabla f( x_{t}) ,d_{t} \rangle _{x_{t}} +\frac{L_{f}}{2}\Vert d_{t}\Vert _{x_{t}}^{2} \notag\\ 
     & =-\eta _{t}\Vert \nabla f( x_{t})\Vert _{x_{t}}^{2} -\langle \nabla f( x_{t}) ,d_{t} -\eta _{t} \nabla f( x_{t}) \rangle _{x_{t}} +\frac{L_{f}}{2}\Vert d_{t}\Vert _{x_{t}}^{2}\notag \\
     & \leqslant -\left( \eta _{t} -\frac{L_{\Phi } G\eta _{t}^{2}}{2}\right)\Vert \nabla f( x_{t})\Vert _{x_t}^{2} +\frac{L_{f}}{2}\Vert d_{t}\Vert _{x_{t}}^{2}.\label{eq:descent1}
    \end{align}
    Here we use Cauchy inequality and Assumption \ref{asp:bounded_gradient} to bound $\langle \nabla f( x_{t}) ,d_{t} -\eta _{t} \nabla f( x_{t}) \rangle _{x_{t}}$ in the second line as
    \begin{align*}
    \langle \nabla f( x_{t}) ,d_{t} -\eta _{t} \nabla f( x_{t}) \rangle _{x_{t}} &\leqslant \Vert \nabla f( x_{t})\Vert _{x_{t}}\Vert d_{t} -\eta _{t} \nabla f( x_{t})\Vert _{x_{t}} \\
    &\leqslant G\Vert d_{t} -\eta _{t} \nabla f( x_{t})\Vert _{x_{t}} \\
    &\leqslant \frac{L_{\Phi } G\eta _{t}^{2}}{2}\Vert \nabla f( x_{t})\Vert _{x_t}^{2}.
    \end{align*}
    By Lemma \ref{lemma1:retraction}, 
    \begin{align*}
    \Vert d_{t}\Vert _{x_{t}}^{2} &=\Vert d_{t} -\eta _{t} \nabla f( x_{t}) +\eta _{t} \nabla f( x_{t})\Vert _{x_{t}}^{2}\\
    & \leqslant 2\Vert d_{t} -\eta _{t} \nabla f( x_{t})\Vert _{x_{t}}^{2} +2\Vert \eta _{t} \nabla f( x_{t})\Vert _{x_{t}}^{2}\\
     & \leqslant \frac{L_{\Phi }^{2} \eta _{t}^{4}}{2}\Vert \nabla f( x_{t})\Vert _{x_{t}}^{4} +2\eta _{t}^{2}\Vert \nabla f( x_{t})\Vert _{x_{t}}^{2}.
    \end{align*}
    Together with (\ref{eq:descent1}) we obtain 
    \begin{equation*}
    f( x_{t+1}) -f( x_{t}) \leqslant -\left( \eta _{t} -\frac{L_{\Phi } G\eta _{t}^{2}}{2} -\eta _{t}^{2} L_{f}\right)\Vert \nabla f( x_{t})\Vert _{x_{t}}^{2} +\frac{L_{f} L_{\Phi }^{2} \eta _{t}^{4}}{4}\Vert \nabla f( x_{t})\Vert _{x_{t}}^{4} .
    \end{equation*} We thus conclude the proof.
\end{proof}

Now we are ready to prove Theorem \ref{thm:rmd}.

\begin{proof}[Proof of Theorem \ref{thm:rmd}]
    If $ \eta < \min\left\{1/\left( L_{\Phi } G/2+L_{f} +L_{f} L_{\Phi }^{2}\right) ,2/G, r/G\right\}$, we have 
\begin{equation*}
\eta ^{4}\Vert \nabla f( x_{t})\Vert _{x_{t}}^{4} \leqslant 4\eta ^{2}\Vert \nabla f( x_{t})\Vert _{x_{t}}^{2} .
\end{equation*}
By Lemma \ref{lemma2:one_step_progress},
\begin{equation*}
f( x_{t+1}) -f( x_{t}) \leqslant -\left( \eta -L_{\Phi } G\eta ^{2} /2-L_{f} \eta ^{2} -L_{f} L_{\Phi }^{2} \eta ^{2}\right)\Vert \nabla f( x_{t})\Vert _{x_{t}}^{2} .
\end{equation*}
Let $ C_{d} =\eta -L_{\Phi } G\eta ^{2} /2-L_{f} \eta ^{2} -L_{f} L_{\Phi }^{2} \eta ^{2}  >0.$

\textbf{Nonconvex Case:} Summing up over $\displaystyle t$ and rearranging, we get 
\begin{align*}
\frac{1}{T}\sum _{t=1}^{T} \| \nabla f(x_{t} )\| _{x_{t}}^{2} & \leqslant \frac{f( x_{1}) -f( x_{T+1})}{C_{d} T} .
\end{align*}

\textbf{Geodesically convex case:} We construct a potential function to prove the convergence rate. Let $ A_{t+1} =A_{t} +a_{t} ,A_{0} =0,\ E_{t} =A_{t}\left( f( x_{t}) -f\left( x^{*}\right)\right), $ where $ a_{t}$ will be specified later. We have 
\begin{align}
E_{t+1} -E_{t} & =A_{t+1}\left( f( x_{t+1}) -f\left( x^{*}\right)\right) -(A_{t+1}-a_t)\left( f( x_{t}) -f\left( x^{*}\right)\right)\notag \\
 & =A_{t+1}( f( x_{t+1}) -f( x_{t})) + a_{t}\left( f( x_{t}) -f\left( x^{*}\right)\right) .\label{eq:potential_descent1}
\end{align}
By the definition of convexity, $f( x_{t}) -f\left( x^{*}\right) \leqslant -\langle \nabla f(x_{t}),\mathrm{Exp}^{-1}_{x_{t}}\left( x^{*}\right) \rangle _{x_{t}}$. Combining with (\ref{eq:potential_descent1}), we have
\begin{align*}
E_{t+1} -E_{t} & \leqslant -C_{d} A_{t+1}\Vert \nabla f( x_{t})\Vert_{x_t} ^{2} -a_{t} \langle \nabla f( x_{t}), \mathrm{Exp}^{-1}_{x_{t}}\left( x^{*}\right) \rangle _{x_{t}}\\
 & =-C_{d} A_{t+1}\left(\Vert \nabla f( x_{t})\Vert_{x_t} ^{2} + \langle \nabla f( x_{t}) ,\frac{a_{t}}{C_{d} A_{t+1}}\mathrm{Exp}^{-1}_{x_{t}}\left( x^{*}\right) \rangle_{x_t} \right) .
\end{align*}
Let $ H( v) =\Vert v\Vert _{x_{t}}^{2} + \langle v,\mathrm{Exp}^{-1}_{x_{t}}\left( x^{*}\right) /C\rangle _{x_{t}}$, where $ C$ is a constant. Since $ H $ is a convex function, it attains its minimal at stationary point $ v^* =  -\frac{\mathrm{Exp}^{-1}_{x_{t}}\left( x^{*}\right)}{2C} $. 
\begin{equation*}
H( v) \geqslant H\left(\frac{\mathrm{Exp}^{-1}_{x_{t}}\left( x^{*}\right)}{2C}\right) =-\frac{1}{4C^{2}}\left\Vert \mathrm{Exp}^{-1}_{x_{t}}\left( x^{*}\right)\right\Vert _{x_{t}}^{2} .
\end{equation*}
By this equation, 
\begin{equation}
E_{t+1} -E_{t} \leqslant \frac{a_{t}^{2}}{4C_{d} A_{t+1}}\left\Vert \mathrm{Exp}^{-1}_{x_{t}}\left( x^{*}\right)\right\Vert _{x_{t}}^{2} .\label{eq:potential_descent2}
\end{equation}
Since
%we assume that the geodesic between any two points in $ A $ exists and is unique, 
we have  $ \left\Vert \mathrm{Exp}^{-1}_{x_{t}}\left( x^{*}\right)\right\Vert_{x_t} =d\left( x_{t} ,x^{*}\right) \leqslant \text{diam}( A)$, plug it into (\ref{eq:potential_descent2}), 
\begin{equation*}
E_{t+1} -E_{t} \leqslant \frac{a_{t}^{2}\text{diam}( A)^{2}}{4C_{d} A_{t+1}} .
\end{equation*}
 Summation with respect to $ t$, we get  
\begin{equation*}
E_{T} \leqslant E_{0} +\frac{\text{diam}( A)^{2}}{4C_{d}}\sum _{t=0}^{T-1}\frac{a_{t}^{2}}{A_{t+1}} .
\end{equation*}
Let  $ A_{t} = t^{2} ,a_{t} =A_{t+1} -A_{t} =2t+1$,  we have 
\begin{align*}
E_{T} &\leqslant E_{0} +\frac{\text{diam}( A)^{2}}{4C_{d}}\sum _{t=0}^{T-1}\frac{( 2t+1)^{2}}{( t+1)^{2}} \\
&< E_{0} +\frac{\text{diam}( A)^{2}}{4C_{d}}\sum _{t=0}^{T-1}\frac{( 2t+2)^{2}}{( t+1)^{2}} \\
&=E_{0} +\frac{\text{diam}( A)^{2} T}{C_{d}} .
\end{align*}
By the definition of $ E $,  $ E_{T} =A_{T}\left( f( x_{T}) -f\left( x^{*}\right)\right)$,
\begin{equation*}
f( x_{T}) -f\left( x^{*}\right) \leqslant \frac{E_{0}}{T^{2}} +\frac{\text{diam}( A)^{2}}{C_{d} T} =\mathcal{O}\left(\frac{1}{\eta T}\right),
\end{equation*}
where the $ \eta $ factor in the denominator comes from $C_d$. This concludes the proof of Theorem \ref{thm:rmd}. 
\end{proof}

To establish the convergence guarantee for stochastic RMD, i.e., Theorem~\ref{thm:SRMD}, we first prove the following lemma. Its role is analogous to Lemma \ref{lemma2:one_step_progress}.

\begin{lemma}\label{lemma3:stoc_one_step_progress}
    Suppose that $ \hat{R}_{x}^t$ satisfies Assumption \ref{asp:Phi} and $f$ satisfies Assumptions \ref{asp:A}-\ref{asp:bounded_gradient}, then 
       \begin{align*}
         \mathbb{E}_{t}f( x_{t+1}) -f( x_{t}) \leqslant -\eta _{t}\Vert \nabla f( x_{t})\Vert _{x_{t}}^{2} +\left(\frac{L_{\Phi } G\eta _{t}^{2}}{2} +\eta _{t}^{2} L_{f}\right)\mathbb{E}_{t}\Vert \nabla f( x_{t} ,\xi _{t})\Vert _{x_{t}}^{2} \\
         \quad\quad\quad\quad\quad+\frac{L_{f} L_{\Phi }^{2} \eta _{t}^{4}}{4}\mathbb{E}_{t}\Vert \nabla f( x_{t} ,\xi _{t})\Vert _{x_{t}}^{4},
    \end{align*}
     where $ \mathbb{E}_{t}[\cdot] = \mathbb{E}[\cdot|x_1,\cdots,x_{t}] $. 
 \end{lemma}

\begin{proof}
    Let $ d_{t} =-\mathrm{Exp}^{}_{x_{t}}(\hat{R}^t_{x_{t}}( -\eta _{t} \nabla f( x_{t},\xi_t)))$, by Assumption \ref{asp:smooth},
    \begin{align}
        \mathbb{E}_{t} f( x_{t+1}) -f( x_{t}) & \leqslant -\mathbb{E}_{t}\langle \nabla f( x_{t}) ,d_{t} \rangle _{x_{t}} +\frac{L_{f}}{2}\mathbb{E}_{t}\Vert d_{t}\Vert _{x_{t}}^{2} \notag\\
         & =-\eta _{t}\Vert \nabla f( x_{t})\Vert _{x_{t}}^{2} -\mathbb{E}_{t} \langle \nabla f( x_{t}) ,d_{t} -\eta _{t} \nabla f( x_{t} ,\xi _{t}) \rangle _{x_{t}} +\frac{L_{f}}{2}\mathbb{E}_{t}\Vert d_{t}\Vert _{x_{t}}^{2} \notag\\
         & \leqslant -\eta _{t}\Vert \nabla f( x_{t})\Vert _{x_t}^{2} +\frac{L_{\Phi } G\eta _{t}^{2}}{2}\mathbb{E}_{t}\Vert \nabla f( x_{t} ,\xi _{t})\Vert _{x_{t}}^{2} +\frac{L_{f}}{2}\mathbb{E}_{t}\Vert d_{t}\Vert _{x_{t}}^{2}, \label{eq:stoc_d}
    \end{align}
    where in the last inequality we use 
    \begin{align*}
        \langle \nabla f( x_{t}) ,d_{t} -\eta _{t} \nabla f( x_{t} ,\xi _{t}) \rangle _{x_{t}} & \leqslant \Vert \nabla f( x_{t})\Vert _{x_{t}}\Vert d_{t} -\eta _{t} \nabla f( x_{t} ,\xi _{t})\Vert _{x_{t}}\\
         & \leqslant G\Vert d_{t} -\eta _{t} \nabla f( x_{t} ,\xi _{t})\Vert _{x_{t}} \\
         &\leqslant \frac{GL_{\Phi } \eta_t^2}{2}\Vert \nabla f( x_{t} ,\xi _{t})\Vert _{x_{t}}^{2} ,
        \end{align*}
     and Lemma \ref{lemma1:retraction}, 
        \begin{align*}
        \Vert d_{t}\Vert _{x_{t}}^{2} & =\Vert d_{t} -\eta _{t} \nabla f( x_{t} ,\xi _{t}) +\eta _{t} \nabla f( x_{t} ,\xi _{t})\Vert _{x_{t}}^{2}\\
         & \leqslant 2\Vert d_{t} -\eta _{t} \nabla f( x_{t} ,\xi _{t})\Vert _{x_{t}}^{2} +2\Vert \eta _{t} \nabla f( x_{t} ,\xi _{t})\Vert _{x_{t}}^{2}\\
         & \leqslant \frac{L_{\Phi }^{2} \eta _{t}^{4}}{2}\Vert \nabla f( x_{t} ,\xi _{t})\Vert _{x_{t}}^{4} +2\eta _{t}^{2}\Vert \nabla f( x_{t} ,\xi _{t})\Vert _{x_{t}}^{2}.
        \end{align*}
    Together with (\ref{eq:stoc_d}), we get 
    \begin{align*}
         \mathbb{E}_{t}f( x_{t+1}) -f( x_{t}) \leqslant -\eta _{t}\Vert \nabla f( x_{t})\Vert _{x_{t}}^{2} +\left(\frac{L_{\Phi } G\eta _{t}^{2}}{2} +\eta _{t}^{2} L_{f}\right)\mathbb{E}_{t}\Vert \nabla f( x_{t} ,\xi _{t})\Vert _{x_{t}}^{2} \\
         \quad\quad\quad\quad\quad+\frac{L_{f} L_{\Phi }^{2} \eta _{t}^{4}}{4}\mathbb{E}_{t}\Vert \nabla f( x_{t} ,\xi _{t})\Vert _{x_{t}}^{4},
    \end{align*} which concludes the proof.
\end{proof} 

Now we are ready to prove Theorem \ref{thm:SRMD}.
 
\begin{proof}[Proof of Theorem \ref{thm:SRMD}]
    Let $\displaystyle C_{sd} =\left( L_{\Phi } G+2L_{f} +2L_{f} L_{\Phi }^{2}\right)$, if $ \eta _{t} = \eta < \min\{1/C_{sd},2/G, r/G\} $, we have 
    \begin{align*}
        \mathbb{E}_t f( x_{t+1}) -f( x_{t}) & \leqslant -\eta\Vert \nabla f( x_{t})\Vert _{x_{t}}^{2} +\left(\frac{L_{\Phi } G\eta^{2}}{2} +\eta^{2} L_{f}\right)\mathbb{E}_t\Vert \nabla f( x_{t} ,\xi _{t})\Vert _{x_{t}}^{2} \\
        & \quad\quad\quad\quad\quad\quad\quad\quad\quad+\frac{L_{f} L_{\Phi }^{2} \eta^{4}}{4}\mathbb{E}_t\Vert \nabla f( x_{t} ,\xi _{t})\Vert _{x_{t}}^{4}\\
         & \leqslant -\eta\Vert \nabla f( x_{t})\Vert _{x_{t}}^{2} + C_{sd} \eta^{2}\mathbb{E}_t\Vert \nabla f( x_{t} ,\xi _{t})\Vert _{x_{t}}^{2},
    \end{align*}
    where we use $ \eta ^{4}\Vert \nabla f( x_{t},\xi_{t})\Vert _{x_{t}}^{4} \leqslant 4\eta ^{2}\Vert \nabla f( x_{t}, \xi_{t})\Vert _{x_{t}}^{2}$ since by assumption, $\Vert \nabla f( x_{t}, \xi_{t})\Vert _{x_{t}}^{2} \leqslant G$. 
    For the second term, we have 
\begin{align*}
\mathbb{E}_{t}\Vert \nabla f( x_{t} ,\xi _{t})\Vert _{x_{t}}^{2} & =\mathbb{E}_{t}\Vert \nabla f( x_{t} ,\xi _{t}) -\nabla f( x_{t}) +\nabla f( x_{t})\Vert _{x_{t}}^{2}\\
 & \leqslant 2\mathbb{E}_{t}\Vert \nabla f( x_{t} ,\xi _{t}) -\nabla f( x_{t})\Vert _{x_{t}}^{2} + 2\Vert \nabla f( x_{t})\Vert _{x_{t}}^{2}\\
 & \leqslant 2\sigma ^{2} +2\Vert \nabla f( x_{t})\Vert _{x_{t}}^{2} .
\end{align*}
Thus, we have 
$$ \mathbb{E}_tf( x_{t+1}) -f( x_{t}) \leqslant -\left( \eta -C_{sd} \eta^{2}\right)\Vert \nabla f( x_{t})\Vert _{x_{t}}^{2} +C_{sd} \eta ^{2} \sigma ^{2}. $$ 
Let $ F_t = \mathbb{E}f(x_{t}) - f(x^*) - {t}C_{sd} \eta ^{2} \sigma ^{2} $, then 
$$ F_{t+1} - F_t \leqslant -\left( \eta -C_{sd} \eta^{2}\right) \mathbb{E}\Vert \nabla f( x_{t})\Vert _{x_{t}}^{2}. $$
Using the same argument as the proof of Theorem \ref{thm:rmd}, we will prove the theorem in different cases. 

\textbf{Nonconvex Case:} Summing up over $t$ and rearranging leads to
    \begin{align*}
 \frac{1}{T}\sum _{t=1}^{T}\mathbb{E} \| \nabla f(x_{t} )\| _{x_{t}}^{2} 
& \leqslant \frac{F_{1} -F_{T+1}}{\left( \eta -C_{sd} \eta ^{2}\right) T} \\
& \leqslant \frac{f( x_{1}) -f( x_{*})}{\left( \eta -C_{sd} \eta ^{2}\right) T} +\frac{C_{sd} \eta \sigma ^{2}}{( 1-C_{sd} \eta )}.
\end{align*}

\textbf{Geodesically convex case:} Let $\displaystyle A_{t} =t^{2} ,\ a_{t} =A_{t} -A_{t-1} =2t+1$. We define $\displaystyle E_{t} =A_{t}\left(\mathbb{E} f( x_{t}) -f\left( x^{*}\right)\right)$. Thus, 
\begin{align*}
E_{t+1} -E_{t} & =A_{t+1}\mathbb{E}[ f( x_{t+1}) -f( x_{t})] +a_{t}\mathbb{E}\left[ f( x_{t}) -f\left( x^{*}\right)\right] .
\end{align*}
By Lemma \ref{lemma3:stoc_one_step_progress} and convexity, we have 
\begin{align*}
E_{t+1} -E_{t} & \leqslant -A_{t+1}\left( \eta -C_{sd} \eta ^{2}\right)\mathbb{E} \| \nabla f( x_{t}) \|_{x_t}^{2} -a_{t} \langle \nabla f(x_{t} ),\mathrm{Exp}^{-1}_{x_{t}}\left( x^{*}\right) \rangle _{x_{t}} +A_{t+1} C_{sd} \eta ^{2} \sigma ^{2}\\
 & \leqslant \frac{a_{t}^{2}}{4A_{t+1}\left( \eta -C_{sd} \eta ^{2}\right)}\mathbb{E} \| \nabla f( x_{t}) \|_{x_t}^{2} +A_{t+1} C_{sd} \eta ^{2} \sigma ^{2}\\
 & \leqslant \frac{a_{t}^{2}\text{diam}( A)}{4A_{t+1}\left( \eta -C_{sd} \eta ^{2}\right)} +A_{t+1} C_{sd} \eta ^{2} \sigma ^{2}
\end{align*}where we use a similar argument as the proof of Theorem \ref{thm:rmd}. Therefore, 
\begin{align*}
E_{T} & \leqslant E_{0} +\frac{\text{diam}( A)}{4\left( \eta -C_{sd} \eta ^{2}\right)}\sum _{t=0}^{T-1}\frac{( 2t+1)^{2}}{( t+1)^{2}} +\sum _{t=0}^{T-1}( t+1)^{2} C_{sd} \eta ^{2} \sigma ^{2}\\
 & \leqslant O\left( E_{0} +\frac{\text{diam}( A) T}{\eta -C_{sd} \eta ^{2}} +C_{sd} \eta ^{2} \sigma T^{3}\right) .
\end{align*}
Note that $\displaystyle E_{T} =T^{2}\left(\mathbb{E} f( x_{T}) -f\left( x^{*}\right)\right)$, we get 
\begin{align*}
\mathbb{E} f( x_{T}) -f\left( x^{*}\right) & \leqslant O\left(\frac{1}{\eta T} +\eta ^{2} \sigma T\right).
\end{align*}
We thus finish the proof of Theorem \ref{thm:SRMD}.
\end{proof}

\section{Numerical Experiments}\label{sec:num}
In this section, we do some experiments on Stiefel manifold optimization to verify our algorithm. We compare our algorithm with the classic algorithm \cite{wen2013feasible}, which is a special case of Algorithm \ref{algo:scd}. Since OptStiefel adapts non-monotonic line search strategy \cite{zhangNonmonotoneLineSearch2004}, which can greatly improve the performance of the algorithm, we adopt the same line search technique \cite{zhangNonmonotoneLineSearch2004} to Algorithm \ref{algo:scd}. In \cite{wen2013feasible}, the authors also use the Barzilai-Borwein (BB) step size heuristic to accelerate their algorithm. However, it is not clear how to use BB step size in the stochastic setting. To highlight the improvement made by randomization, we do not use BB step size in our experiment. We set the stop criterion to be 
\begin{equation*}
  \| \nabla_\mathcal{M}f(x) \| \leqslant 10^{-5} \ \text{or} \quad \|x_t - x_{t-1}\|\leqslant 10^{-5} \ \text{or}\quad  \|f(x_t) - f(x_{t-1})\| \leqslant 10^{-8}.
\end{equation*}
All the experiments are tested on a machine with Intel(R) Core(TM) i7-9750H CPU @ 2.60GHz   2.59 GHz and 16.0GB RAM\footnote{The code and data set are available at https://github.com/JiyuanTan/RMD/tree/main.}.

\subsection{Linear Eigenvalue Problem}
Given a symmetric matrix $A$, the linear eigenvalue problem is 
\begin{align*}
    \max_{X\in\mathrm{St}(n,p)} \text{Tr}(X^\T A X) = \sum_{k=1}^p \lambda_k,
\end{align*}
where $\lambda_1\geq\lambda_2\geq\cdots\geq\lambda_p$ are the $p$th largest eigenvalues of $A$.  In the experiments, we generate a random matrix $N$ whose entries are i.i.d. standard normal random variables and set $A=N^\top N$.
The test results are reported in Table~\ref{tab:eig}. We observe that while both algorithms produce solutions of similar quality, SCGD achieves a lower runtime than CGD. One reason is that SCGD only needs to solve several smaller linear systems rather than a large one in each iteration.

\begin{table}[!htb]
    \centering
    \begin{tabular}{|c|c|c|c|c|}\hline
        Algorithm & $n $ &  1000& 2000& 5000 \\ \hline 
        \multirow{2}{*}{CGD} & Time & 1.41&7.10 & 64.67\\\cline{2-5}
        &Error & 1.20e-06	& 9.28e-07 &	5.88e-06 \\\hline
        \multirow{2}{*}{SCGD} & Time &0.72&	4.56&	38.10 \\\cline{2-5}
        &Error &  7.08e-08 &	9.54e-07&	8.09e-06\\\hline
    \end{tabular}
    \caption{Results of the two algorithms. When $ K =1 $ the step size is $ 10^{-3} $. When $ K > 1 $ the step size is $ 10^{-2} $. We take $p = 10$. The number of block $K$ in Algorithm \ref{algo:scd} is chosen to be $ \lfloor n/300 \rfloor $. The results are the average of 5 runs. }
    \label{tab:eig}
\end{table}

\subsection{Orthogonal Procrustes Problem}
The orthogonal procrustes problem \cite{gower2004procrustes} is 
\begin{align*}
    \min_{X\in\mathrm{St}(n,p)} & \|AX-B\|_F^2.
\end{align*}
In the experiments, we generate $A=(A_{ij})$ with $A_{ij}\overset{\text{i.i.d.}}{\sim}\mathrm{Unif}(0,1)$, and construct $X^\star\in\mathrm{St}(n,p)$ by taking the $Q$ factor of the QR decomposition of a Gaussian random matrix.
We then set $B=AX^\star$. Since the objective is nonnegative and $X^\star$ attains value $0$, the optimal value is $0$.
The results are reported in Table~\ref{tab:procrustes}.

\begin{table}[!htb]
    \centering
    \begin{tabular}{|c|c|c|c|c|}\hline
        Algorithm & $n $ &  1000& 2000& 5000 \\ \hline 
        \multirow{2}{*}{CGD} & Time & 11.34&	37.68&	341.41\\\cline{2-5}
        &Error & 7.55e-2&	7.30e-2 &	1.13e-1 \\\hline
        \multirow{2}{*}{SCGD} & Time & 11.58&	38.49&	211.09\\\cline{2-5}
        &Error &  2.05e-2 &	4.41e-2&	9.76e-2\\\hline
    \end{tabular}
    \caption{Results of the two algorithms. When $ K =1 $ the step size is $ 10^{-3} $. When $ K > 1 $ the step size is $ 10^{-2} $. We take $p = 10$. The number of block $K$ in Algorithm \ref{algo:scd} is chosen to be $ \lfloor n/300 \rfloor $. }
    \label{tab:procrustes}
\end{table}

\section{Conclusion}\label{sec:con}

In this paper, we propose a new framework for mirror descent (MD) on Riemannian manifolds, based on the key observation that MD can be interpreted as an optimization method under a suitable reparameterization.  Under mild assumptions, we obtain the first non-asymptotic convergence result of Riemannian Mirror Descent in both the deterministic and stochastic settings. Furthermore, we introduce the Stochastic Curvilinear Gradient Descent algorithm for large-scale Stiefel manifold optimization under the RMD framework, which we believe is of independent interest. Similar techniques may be applied to address optimization problems in other settings. 

An interesting direction for future work is the explicit construction of problem-adapted reparameterizations. For instance, Li et al.~\cite{li2022implicit} showed that, under any commuting parametrization, the gradient flow is equivalent to the continuous-time mirror flow. In our general Riemannian setting, where no additional global structure is imposed, the resulting reparameterizations are inherently local. It would be of interest to investigate whether stronger geometric assumptions, such as those satisfied by Hessian manifolds, permit global reparameterizations and lead to sharper guarantees and improved algorithms.

%Further investigation is needed in this direction to identify and develop efficient reparameterization techniques for various optimization scenarios. 

\bibliographystyle{plain}
\bibliography{ref.bib}

@article{raskutti2015information,
  title={The information geometry of mirror descent},
  author={Raskutti, Garvesh and Mukherjee, Sayan},
  journal={IEEE Transactions on Information Theory},
  volume={61},
  number={3},
  pages={1451--1457},
  year={2015},
  publisher={IEEE}
}

@inproceedings{gunasekar2021mirrorless,
  title={Mirrorless mirror descent: A natural derivation of mirror descent},
  author={Gunasekar, Suriya and Woodworth, Blake and Srebro, Nathan},
  booktitle={International Conference on Artificial Intelligence and Statistics},
  pages={2305--2313},
  year={2021},
  organization={PMLR}
}

@article{lei2020convergence,
  title={Convergence of online mirror descent},
  author={Lei, Yunwen and Zhou, Ding-Xuan},
  journal={Applied and Computational Harmonic Analysis},
  volume={48},
  number={1},
  pages={343--373},
  year={2020},
  publisher={Elsevier}
}

@article{vandereycken2013low,
  title={Low-rank matrix completion by {R}iemannian optimization},
  author={Vandereycken, Bart},
  journal={SIAM Journal on Optimization},
  volume={23},
  number={2},
  pages={1214--1236},
  year={2013},
  publisher={SIAM}
}

@article{sato2019riemannian,
  title={Riemannian stochastic variance reduced gradient algorithm with retraction and vector transport},
  author={Sato, Hiroyuki and Kasai, Hiroyuki and Mishra, Bamdev},
  journal={SIAM Journal on Optimization},
  volume={29},
  number={2},
  pages={1444--1472},
  year={2019},
  publisher={SIAM}
}

@article{zhang2016riemannian,
  title={Riemannian {SVRG}: Fast stochastic optimization on {R}iemannian manifolds},
  author={Zhang, Hongyi and J Reddi, Sashank and Sra, Suvrit},
  journal={Advances in Neural Information Processing Systems},
  pages={4592--4600},
  volume={29},
  year={2016}
}

@inproceedings{alimisis2021momentum,
  title={Momentum improves optimization on {R}iemannian manifolds},
  author={Alimisis, Foivos and Orvieto, Antonio and Becigneul, Gary and Lucchi, Aurelien},
  booktitle={International Conference on Artificial Intelligence and Statistics},
  pages={1351--1359},
  year={2021},
  organization={PMLR}
}

@article{srinivasan2022sufficient,
  title={Sufficient conditions for non-asymptotic convergence of {R}iemannian optimisation methods},
  author={Srinivasan, Vishwak and Wilson, Ashia},
  journal={14th Annual Workshop on Optimization for Machine Learning},
  year={2022}
}

@inproceedings{tripuraneni2018averaging,
  title={Averaging stochastic gradient descent on {R}iemannian manifolds},
  author={Tripuraneni, Nilesh and Flammarion, Nicolas and Bach, Francis and Jordan, Michael I},
  booktitle={Conference On Learning Theory},
  pages={650--687},
  year={2018},
  organization={PMLR}
}

@inproceedings{zhang2016first,
  title={First-order methods for geodesically convex optimization},
  author={Zhang, Hongyi and Sra, Suvrit},
  booktitle={Conference on Learning Theory},
  pages={1617--1638},
  year={2016},
  organization={PMLR}
}

@article{bansal2018can,
  title={Can we gain more from orthogonality regularizations in training deep networks?},
  author={Bansal, Nitin and Chen, Xiaohan and Wang, Zhangyang},
  journal={Advances in Neural Information Processing Systems},
  pages={4261--4271},
  volume={31},
  year={2018}
}

@inproceedings{arjovsky2016unitary,
  title={Unitary evolution recurrent neural networks},
  author={Arjovsky, Martin and Shah, Amar and Bengio, Yoshua},
  booktitle={International Conference on Machine Learning},
  pages={1120--1128},
  year={2016},
  organization={PMLR}
}

@inproceedings{tan2014riemannian,
  title={Riemannian pursuit for big matrix recovery},
  author={Tan, Mingkui and Tsang, Ivor W and Wang, Li and Vandereycken, Bart and Pan, Sinno Jialin},
  booktitle={International Conference on Machine Learning},
  pages={1539--1547},
  year={2014},
  organization={PMLR}
}

@article{cherian2016riemannian,
  title={Riemannian dictionary learning and sparse coding for positive definite matrices},
  author={Cherian, Anoop and Sra, Suvrit},
  journal={IEEE Transactions on Neural Networks and Learning Systems},
  volume={28},
  number={12},
  pages={2859--2871},
  year={2016},
  publisher={IEEE}
}

@article{bento2017iteration,
  title={Iteration-complexity of gradient, subgradient and proximal point methods on {R}iemannian manifolds},
  author={Bento, Glaydston C and Ferreira, Orizon P and Melo, Jefferson G},
  journal={Journal of Optimization Theory and Applications},
  volume={173},
  number={2},
  pages={548--562},
  year={2017},
  publisher={Springer}
}

@article{wen2013feasible,
  title={A feasible method for optimization with orthogonality constraints},
  author={Wen, Zaiwen and Yin, Wotao},
  journal={Mathematical Programming},
  volume={142},
  number={1-2},
  pages={397--434},
  year={2013},
  publisher={Springer}
}

@inproceedings{duchi2010composite,
  title={Composite objective mirror descent.},
  author={Duchi, John C and Shalev-Shwartz, Shai and Singer, Yoram and Tewari, Ambuj},
  booktitle={Conference on Learning Theory},
  volume={10},
  pages={14--26},
  year={2010},
  organization={Citeseer}
}

@book{nemirovski1983problem,
  title={{Problem Complexity and Method Efficiency in Optimization}},
  author={Nemirovski, Arkadi Semenovi{\v{c}} and Yudin, David Borisovich},
  year={1983},
  publisher={Wiley-Interscience}
}

@article{beck2003mirror,
  title={Mirror descent and nonlinear projected subgradient methods for convex optimization},
  author={Beck, Amir and Teboulle, Marc},
  journal={Operations Research Letters},
  volume={31},
  number={3},
  pages={167--175},
  year={2003},
  publisher={Elsevier}
}

@article{amid2020reparameterizing,
  title={Reparameterizing mirror descent as gradient descent},
  author={Amid, Ehsan and Warmuth, Manfred KK},
  journal={Advances in Neural Information Processing Systems},
  volume={33},
  pages={8430--8439},
  year={2020}
}

@article{nemirovski2009robust,
  title={Robust stochastic approximation approach to stochastic programming},
  author={Nemirovski, Arkadi and Juditsky, Anatoli and Lan, Guanghui and Shapiro, Alexander},
  journal={SIAM Journal on Optimization},
  volume={19},
  number={4},
  pages={1574--1609},
  year={2009},
  publisher={SIAM}
}

@article{bregman1967relaxation,
  title={The relaxation method of finding the common point of convex sets and its application to the solution of problems in convex programming},
  author={Bregman, Lev M},
  journal={USSR Computational Mathematics and Mathematical Physics},
  volume={7},
  number={3},
  pages={200--217},
  year={1967},
  publisher={Elsevier}
}

@article{li2022implicit,
  title={Implicit bias of gradient descent on reparametrized models: On equivalence to mirror descent},
  author={Li, Zhiyuan and Wang, Tianhao and Lee, Jason D and Arora, Sanjeev},
  journal={Advances in Neural Information Processing Systems},
  volume={35},
  pages={34626--34640},
  year={2022}
}

@article{zhangNonmonotoneLineSearch2004,
  title = {A Nonmonotone Line Search Technique and Its Application to Unconstrained Optimization},
  author = {Zhang, Hongchao and Hager, William W.},
  year = {2004},
  month = jan,
  journal = {SIAM Journal on Optimization},
  volume = {14},
  number = {4},
  pages = {1043--1056},
  issn = {1052-6234, 1095-7189},
  doi = {10.1137/S1052623403428208},
  urldate = {2024-03-01},
  langid = {english}
}

@book{gower2004procrustes,
  title={Procrustes {P}roblems},
  author={Gower, John C and Dijksterhuis, Garmt B},
  volume={30},
  year={2004},
  publisher={OUP Oxford}
}

@inproceedings{yang2022policy,
  title={Policy optimization with stochastic mirror descent},
  author={Yang, Long and Zhang, Yu and Zheng, Gang and Zheng, Qian and Li, Pengfei and Huang, Jianhang and Pan, Gang},
  booktitle={Proceedings of the AAAI Conference on Artificial Intelligence},
  volume={36},
  pages={8823--8831},
  year={2022}
}

@article{zhan2023policy,
  title={Policy mirror descent for regularized reinforcement learning: A generalized framework with linear convergence},
  author={Zhan, Wenhao and Cen, Shicong and Huang, Baihe and Chen, Yuxin and Lee, Jason D and Chi, Yuejie},
  journal={SIAM Journal on Optimization},
  volume={33},
  number={2},
  pages={1061--1091},
  year={2023},
  publisher={SIAM}
}

@article{sun2022mirror,
  title={Mirror descent maximizes generalized margin and can be implemented efficiently},
  author={Sun, Haoyuan and Ahn, Kwangjun and Thrampoulidis, Christos and Azizan, Navid},
  journal={Advances in Neural Information Processing Systems},
  volume={35},
  pages={31089--31101},
  year={2022}
}

@article{soh2023mirror,
  title={Mirror descent of {H}opfield model},
  author={Soh, Hyungjoon and Kim, Dongyeob and Hwang, Juno and Jo, Junghyo},
  journal={Neural Computation},
  volume={35},
  number={9},
  pages={1529--1542},
  year={2023},
  publisher={MIT Press One Rogers Street, Cambridge, MA 02142-1209, USA journals-info~…}
}

@inproceedings{criscitiello2022negative,
  title={Negative curvature obstructs acceleration for strongly geodesically convex optimization, even with exact first-order oracles},
  author={Criscitiello, Christopher and Boumal, Nicolas},
  booktitle={Conference on Learning Theory},
  pages={496--542},
  year={2022},
  organization={PMLR}
}

@article{shu2025revisit,
  title={Revisit First-order Methods for Geodesically Convex Optimization},
  author={Yunlu Shu and Jiaxin Jiang and Lei Shi and Tianyu Wang},
  journal={arXiv preprint arXiv:2504.06814},
  year={2025}
}

\end{document}